%% file: sgfs.tex
\documentclass{article}

\usepackage{algorithm}
\usepackage{algorithmic}

\usepackage{graphicx, xcolor, subfigure, mathrsfs}
\usepackage{epsfig}
\usepackage{wrapfig}
\usepackage{pict2e, slashbox}
\usepackage{textcomp}
\usepackage{multirow}
\usepackage{url}

\usepackage{latexsym, amssymb, bm, amsthm, amsmath}

\newcommand{\comment}[1]{}

\newcommand{\h}[1]{\widehat{#1}}

\newcommand{\mcal}[1]{\mathcal{#1}}

\newtheorem{lemma}{Lemma}
\newtheorem{assump}{Assumption}
\newtheorem{theorem}{Theorem}
\newtheorem{corollary}{Corollary}

\usepackage{hyperref}

\begin{document}

\title{Efficient Sparse Group Feature Selection via \\ Nonconvex Optimization}
\author{Shuo Xiang$^{1,2}$  \qquad Xiaotong Shen$^3$  \qquad Jieping Ye$^{1,2}$ \\
  \normalsize $^1$Department of Computer Science and Engineering, Arizona State University, AZ 85287\\
  \normalsize $^2$Center for Evolutionary Medicine and Informatics, The Biodesign Institute, Arizona State University, AZ 85287\\
  \normalsize $^3$School of Statistics, University of Minnesota, Minneapolis, MN
  55455 }
\maketitle

%
%

\begin{abstract}
  Sparse feature selection has been demonstrated to be effective in handling
  high-dimensional data.  While promising, most of the existing works use convex
  methods, which may be suboptimal in terms of the accuracy of feature selection
  and parameter estimation. In this paper, we expand a nonconvex paradigm to
  sparse group feature selection, which is motivated by applications that
  require identifying the underlying group structure and performing feature
  selection simultaneously. The main contributions of this article are twofold:
  (1) statistically, we introduce a nonconvex sparse group feature selection
  model which can reconstruct the oracle estimator. Therefore, consistent
  feature selection and parameter estimation can be achieved; (2)
  computationally, we propose an efficient algorithm that is applicable to
  large-scale problems. Numerical results suggest that the proposed nonconvex
  method compares favorably against its competitors on synthetic data and
  real-world applications, thus achieving desired goal of delivering high
  performance.
\end{abstract}

\section{Introduction}\label{sec:introduction}

During the past decade, sparse feature selection has been extensively
investigated, on both optimization algorithms~\cite{bach2010convex} and
statistical properties~\cite{zhao2006model, tibshirani1996regression,
  bickel2009simultaneous}.  When the data possesses certain group structure,
sparse modeling has been explored in~\cite{yuan2006model, meier2008group,
  huang2010benefit} for group feature selection. The group
lasso~\cite{yuan2006model} proposes an $L_2$-regularization method for each
group, which ultimately yields a group-wisely sparse model. The utility of such
a method has been demonstrated in detecting splice
sites~\cite{yang2010online}---an important step in gene finding and
theoretically justified in~\cite{huang2010benefit}. The sparse group lasso
~\cite{friedman2010note} enables to encourage sparsity at the level of both
features and groups simultaneously.

In the literature, most approaches use convex methods due to globality of the
solution and tractable computation. However, this may lead to suboptimal
results. Recent studies demonstrate that nonconvex methods, for instance, the
truncated $L_1$-penalty~\cite{shen2012likelihood, mazumder2011sparsenet,
  zhang2011multi}, may have potential to deliver superior performance than the
standard $L_1$-formulation. In addition,~\cite{shen2012likelihood} suggests that
a constrained version of nonconvex regularization is slightly more preferable
than its regularization counterpart due to theoretical merits.

In this article, we investigate the sparse group feature selection (SGFS)
through a constrained nonconvex formulation. Ideally, we wish to optimize the
following $L_0$-model:
\begin{equation}
  \label{eq:csglp:nonconvex}
  \begin{aligned}
 &\underset{\bm x}{\text{minimize}} && \frac{1}{2}\|\bm A \bm x - \bm y\|^2_2 \\
   &\text{subject to} && \sum_{j=1}^p I(|x_j| \neq 0)  \le s_1 \\
   & && \sum_{j=1}^{|G|}I(\|\bm x_{G_j}\|_2 \neq 0) \le s_2,
  \end{aligned}
\end{equation}
where $\bm A$ is an $n$ by $p$ data matrix with its columns representing
different features. $\bm x=(x_1,\cdots,x_p)$ is partitioned into $|G|$ non-overlapping groups $ \{ \bm x_{G_i} \}$
and $I(\cdot)$ is the indicator function.
The advantage of the $L_0$-model~\eqref{eq:csglp:nonconvex} lies in its complete
control on two levels of sparsity $(s_1,s_2)$, which are the numbers of features
and groups respectively.  However, a problem
like~\eqref{eq:csglp:nonconvex} is known to be
NP-hard~\cite{natarajan1995sparse}.

This paper develops an efficient nonconvex method, which is a computational
surrogate of the $L_0$-method described above and has theoretically guaranteed
performance. We contribute in two aspects: (i) statistically, the proposed
method retains the merits of the $L_0$ approach~\eqref{eq:csglp:nonconvex} in
the sense that the oracle estimator can be reconstructed, which leads to
consistent feature selection and parameter estimation; (ii) computationally, our
efficient optimization tool enables to treat large-scale problems.



\section{Nonconvex Formulation and Computation}\label{sec:formulation}

 One major difficulty of solving~\eqref{eq:csglp:nonconvex} comes from
nonconvex and discrete constraints, which require enumerating all
possible combinations of features and groups to achieve the optimal
solution. Therefore we approximate these constraints by their
continuous computational surrogates:
\begin{equation}
  \label{eq:csglp1}
  \begin{aligned}
 &\underset{\bm x}{\text{minimize}} && \frac{1}{2}\|\bm A \bm x - \bm y\|^2_2 \\
   &\text{subject to} && \sum_{j=1}^p J_{\tau}(|x_j|) \le s_1, \\
   & &&              \sum_{i=1}^{|G|} J_{\tau}(\|\bm x_{G_i}\|_2) \le s_2,
\end{aligned}
\end{equation}
where $J_{\tau}(z)=\min(|z|/\tau,1)$ is a truncated $L_1$-function
approximating the $L_0$-function~\cite{shen2012likelihood, zhang2010analysis},
and $\tau>0$ is a tuning parameter such that $J_{\tau}(z)$ approximates
the indicator function $I(|z| \neq 0)$ as $\tau$ approaches zero.

  To solve the nonconvex problem~\eqref{eq:csglp1}, we develop
a Difference of Convex (DC) algorithm based on a decomposition
of each nonconvex constraint function into a difference of two
convex functions; for instance,
\begin{equation*}
\sum_{j=1}^p J_{\tau}(|x_j|)=
S_1(\bm x)-S_2(\bm x),
\end{equation*}
where
\[
S_1(\bm x)= \frac{1}{\tau} \sum_{j=1}^p |x_j|
\]
and
\[
S_2(\bm
x)=\frac{1}{\tau} \sum_{j=1}^p \max \{|x_j|- \tau, 0 \}
\]
are convex in $\bm x$. Then each trailing convex function, say
$S_2(\bm x)$, is replaced by its affine
minorant at the previous iteration
\begin{equation}
  \label{eq:affine:minorant}
  S_1(\bm x) - S_2(\hat{\bm x}^{(m-1)}) - \nabla S_2(\hat{\bm x}^{(m-1)})
^T(\bm x - \hat{\bm x}^{(m-1)}),
\end{equation}
which yields an upper approximation of the constraint function
$\sum_{j=1}^p J_{\tau}(|x_j|)$ as follows:
\begin{equation}\label{eq:constraint1}
\frac{1}{\tau} \sum_{j=1}^p |x_j|\cdot I(|\hat{x}^{(m-1)}_j|
\leq \tau) + \sum_{j=1}^p I(|\hat{x}^{(m-1)}_j| > \tau) \le s_1.
\end{equation}
Similarly, the second nonconvex constraint in~\eqref{eq:csglp1} can be
approximated by
\begin{equation}\label{eq:constraint2}
   \frac{1}{\tau} \sum_{j=1}^{|G|} \|\bm x_{G_j}\|_2\cdot I(\|\hat{\bm x}^{(m-1)}_{G_j}\|_2 \leq \tau) + \sum_{j=1}^{|G|}I(\|\hat{\bm x}^{(m-1)}_{G_j}\|_2 > \tau) \le s_2.
\end{equation}
Note that both~\eqref{eq:constraint1} and~\eqref{eq:constraint2} are convex constraints, which result in a convex subproblem as follows:
\begin{equation}\label{eq:lm-prime}
 \begin{aligned}
 &\underset{\bm x}{\text{minimize}} && \frac{1}{2}\|\bm A \bm x - \bm y\|^2_2 \\
 &\text{subject to} && \frac{1}{\tau} \|\bm x^{T_1(\hat{\bm x}^{(m-1)})}\|_1 \leq s_1- (p-|T_1(\hat{\bm x}^{(m-1)})|)  \\
 & && \frac{1}{\tau} \|\bm x^{T_3(\hat{\bm x}^{(m-1)})}\|_G \leq s_2 - (|G|-|T_2(\hat{\bm x}^{(m-1)})|), \\
 \end{aligned}
\end{equation}
where $T_1$, $T_2$ and $T_3$ are the support sets\footnote{Support
  sets indicate that the elements outside these sets have no effect on the particular
  items in the constraints of~\eqref{eq:lm-prime}.} defined as:
\begin{equation*}
  \begin{aligned}
  T_1(\bm x)&=\{i:|x_i|\le \tau \}\\
  T_2(\bm x)&=\{i:\|x_{G_i}\|_2 \le \tau \} \\
  T_3(\bm x)&=\{i: x_i\in \bm x_{G_j}, j\in T_2(\bm x)\},
  \end{aligned}
\end{equation*}
$\|\bm x^{T_1}\|_1$ and $\|\bm x^{T_3}\|_G$ denote the corresponding value
restricted on $T_1$ and $T_3$ respectively, and $\|\bm x\|_G =
\sum_{i=1}^{|G|}\|\bm x_{G_i}\|_2$. Solving~\eqref{eq:lm-prime} would provide us
an updated solution, denoted as $\hat{\bm x}^{(m)}$. Such procedure is iterated
until the objective value is no longer decreasing, indicating that a local
minimizer is achieved. The DC algorithm is summarized in Algorithm~\ref{alg:DC},
from which we can see that efficient computation of~\eqref{eq:lm-prime} is
critical to the overall DC routine. We defer detailed discussion of this part to
Section~\ref{sec:opt}.

\begin{algorithm}[thb]
    \caption{DC programming for solving~\eqref{eq:csglp1}}
    \label{alg:DC}
 \begin{algorithmic}[1]
    \REQUIRE $\bm A$, $\bm y$, $s_1$, $s_2$
    \ENSURE \texttt{solution} $\bm x$ to~\eqref{eq:csglp1}
    \STATE (\textbf{Initialization}) \texttt{Initialize} $\hat{\bm x}^{(0)}$.
\STATE  (\textbf{Iteration}) \texttt{At iteration} $m$, \texttt{compute} $\hat{\bm x}^{(m)}$ \texttt{by optimizing~\eqref{eq:lm-prime}}.
\STATE  (\textbf{Stopping Criterion}) \texttt{Terminate
when the objective function stops decreasing.}
 \end{algorithmic}
\end{algorithm}

\section{Theoretical Results}\label{sec:theory}
\input{theory.tex}

\section{Optimization Procedures}\label{sec:opt}
As mentioned in Section~\ref{sec:formulation}, efficient computation of the
convex subproblem~\eqref{eq:lm-prime} is of critical importance for the proposed
DC algorithm. Note that~\eqref{eq:lm-prime} has an identical
form of the constrained sparse group lasso problem:
\begin{equation} \label{eq:csglp:convex}
  \begin{aligned}
   &\underset{\bm x}{\text{minimize}} && \frac{1}{2}\|\bm A \bm x - \bm y\|^2_2 &&\\
    &\text{subject to} && \|\bm x\|_1 \le s_1 &&\\
   &                   && \|\bm x\|_G \le s_2 &&
  \end{aligned}
\end{equation}
except that $\bm x$ is restricted to the two support sets. As to be shown in
Section~\ref{subsec:equi:supp}, an algorithm for
solving~\eqref{eq:lm-prime} can be obtained through only a few modifications on
that of~\eqref{eq:csglp:convex}. Therefore, we first focus on
solving~\eqref{eq:csglp:convex}.

\subsection{Accelerated Gradient Method}
\label{sec:accelerated}
For large-scale problems, the dimensionality of data can be very high, therefore
first-order optimization is often preferred. We adapt the well-known accelerated
gradient method (AGM)~\cite{nesterov2007gradient, beck2009fast}, which is
commonly used due to its fast convergence rate.

To apply AGM to our formulation~\eqref{eq:csglp:convex}, the crucial step is to
solve the following Sparse Group Lasso Projection (SGLP):
\begin{equation}
  \label{eq:sglp}
  \begin{aligned}
   &\underset{\bm x}{\text{minimize}} && \frac{1}{2}\|\bm x - \bm v\|^2_2 &&\\
    &\text{subject to} && \|\bm x\|_1 \le s_1 &&\quad(C_1)\\
   &                   && \|\bm x\|_G \le s_2 &&\quad(C_2),
  \end{aligned}
\end{equation}
which is an Euclidean projection onto a convex set and a special case
of~\eqref{eq:csglp:convex} when $A$ is the identity.  For convenience, let $C_1$
and $C_2$ denote the above two constraints in what follows.

Since the AGM is a standard framework whose efficiency mainly depends on that of
the projection step, we leave the detailed description of AGM in the
supplement and introduce the efficient algorithm for this projection
step~\eqref{eq:sglp}.

\subsection{Efficient Projection}
We begin with some special cases of~\eqref{eq:sglp}. If only $C_1$
exists,~\eqref{eq:sglp} becomes the well-known $L_1$-ball
projection~\cite{duchi2008efficient}, whose optimal solution is denoted as
$\mcal{P}_1^{s_1}(\bm v)$, standing for the projection of $\bm v$ onto the
$L_1$-ball with radius $s_1$. On the other hand, if only $C_2$ is involved, it
becomes the group lasso projection, denoted as $\mcal{P}_G^{s_2}$. Moreover, we
say a constraint is~\emph{active}, if and only if an equality holds at the
optimal solution $x^*$; otherwise, it is~\emph{inactive}.

 Preliminary results are summarized in Lemma~\ref{lm:preliminary}:
\begin{lemma}\label{lm:preliminary}
Denote a global minimizer of~\eqref{eq:sglp} as $\bm x^*$. Then the
following results hold:
\begin{enumerate}
\item If both $C_1$ and $C_2$ are inactive, then $\bm x^*= \bm v$.
\item If $C_1$ is the only active constraint, i.e., $\|\bm x^*\|_1=s_1$, $\|\bm x^*\|_G <
  s_2$, then $\bm x^* = \mcal{P}^{s_1}_1(\bm v)$
\item If $C_2$ is the only active constraint, i.e., $\|\bm x^*\|_1 < s_1$, $\|\bm x^*\|_G
  = s_2$, then $\bm x^* = \mcal{P}^{s_2}_G(\bm v)$

\end{enumerate}
\end{lemma}

\subsubsection{Computing $x^*$ from the optimal dual variables} 
\label{sec:computex}

Lemma~\ref{lm:preliminary} describes a global minimizer when either
constraint is inactive. Next we consider the case in which
both $C_1$ and $C_2$ are active. By the convex duality
theory~\cite{boyd2004convex}, there exist unique non-negative dual
variables $\lambda^*$ and $\eta^*$ such that $x^*$ is also
the global minimizer of the following regularized problem:
\begin{equation}
 \label{eq:sgl}
 \begin{aligned}
   &\underset{\bm x}{\text{minimize}} && \frac{1}{2}\|\bm x - \bm v\|^2_2 +
   \lambda^* \|\bm x\|_1 + \eta^* \|\bm x\|_G,
 \end{aligned}
\end{equation}
whose solution is given by the following Theorem.
\begin{theorem}[\cite{friedman2010note}]\label{thm:sgl}
The optimal solution $\bm x^*$ of~\eqref{eq:sgl} is given by
 \begin{equation}
   \begin{aligned}  \label{eq:sgl:sol}
     \bm x^*_{G_i} &= \max\{\|\bm v_{G_i}^{\lambda^*}\|_2-\eta^*, 0\}\frac{\bm v_{G_i}^{\lambda^*}}{\|\bm v_{G_i}^{\lambda^*}\|_2}\quad i = 1, 2, \cdots, |G| \\
   \end{aligned}
 \end{equation}
 where $\bm v_{G_i}^{\lambda^*}$ is computed via soft-thresholding~\cite{donoho2002noising} $\bm v_{G_i}$ with threshold $\lambda^*$ as follows:
\[
\bm v_{G_i}^{\lambda^*} = \text{SGN}(\bm v_{G_i})\cdot\max\{|\bm v_{G_i}| - \lambda^*, 0\},
\]
where $\text{SGN}(\cdot)$ is the sign function and all the operations are taken element-wisely.
\end{theorem}
Theorem~\ref{thm:sgl} gives an analytical solution of $\bm x^*$ in an ideal
situation when the values of $\lambda^*$ and $\eta^*$ are given. Unfortunately, this
is not the case and the values of $\lambda^*$ and $\eta^*$ need to be computed
directly from~\eqref{eq:sglp}. Based on Theorem~\ref{thm:sgl}, we have the
following conclusion characterizing the relations between the dual variables:
\begin{corollary}
The following equations hold:
\begin{eqnarray}
 \|\bm x^*\|_1 = \sum_{i=1}^{|G|}\max\{\|\bm v_{G_i}^{\lambda^*}\|_2-\eta^*,0\}\frac{\|\bm v_{G_i}^{\lambda^*}\|_1}{\|\bm v_{G_i}^{\lambda^*}\|_2}=s_1 & &\label{eq:c1}\\
 \|\bm x^*\|_G = \sum_{i=1}^{|G|}\max\{\|\bm v_{G_i}^{\lambda^*}\|_2 - \eta^*,0\}=s_2 & &. \label{eq:c2}
\end{eqnarray}
\end{corollary}

Suppose $\lambda^*$ is given, then computing $\eta^*$ from~\eqref{eq:c2} amounts
to solving a median finding problem, which can be done in linear
time~\cite{duchi2008efficient}.

Finally, we treat the case of unknown $\lambda^*$ (thus unknown $\eta^*$). We propose
an efficient bisection approach to compute it.

\subsubsection{Computing $\lambda^*$: bisection}\label{sec:computelambda}

Given an initial guess (estimator) of $\lambda^*$, says $\hat{\lambda}$, one may
perform bisection to locate the optimal $\lambda^*$, provided that there exists
an oracle procedure indicating if the optimal value is greater than
$\hat{\lambda}$\footnote{An upper bound and a lower bound of $\lambda^*$ should
 be provided in order to perform the bisection. These bounds can be easily
 derived from the assumption that both $C_1$ and $C_2$ are active.}.  This
bisection method can estimate $\lambda^*$ in logarithm time. Next, we shall design
an oracle procedure.


 Let the triples
\[
(\bm x^*, \lambda^*, \eta^*) = \text{SGLP}(\bm v, s_1, s_2)
\]
be the optimal solution of~\eqref{eq:sglp} with both
constraints active, i.e., $\|\bm x^*\|_1=s_1$, $\|\bm x^*\|_G=s_2$, with $(\lambda^*,
\eta^*)$ be the optimal dual variables.  Consider the following two sparse
group lasso projections:
\begin{equation*}
 \begin{aligned}
   (\bm x, \lambda, \eta) &= \text{SGLP}(\bm v, s_1, s_2), \\
   (\bm x', \lambda', \eta') &= \text{SGLP}(\bm v, s_1', s_2').
 \end{aligned}
\end{equation*}
The following key result holds.

\begin{theorem}\label{thm:monotone}
If $\lambda \le \lambda'$ and $s_2 =s_2'$,  then $s_1 \ge
 s_1'$.
\end{theorem}

Theorem~\ref{thm:monotone} gives the oracle procedure with its proof presented in the
supplement. For a given estimator $\hat{\lambda}$, we compute its
corresponding $\hat{\eta}$ from~\eqref{eq:c2} and then $\hat{s_1}$
from~\eqref{eq:c1}, satisfying $(\hat{\bm x}, \hat{\lambda}, \hat{\eta}) =
\text{SGLP}(\bm v, \hat{s_1}, s_2)$. Then $\hat{s_1}$ is compared with
$s_1$. Clearly, by Theorem~\ref{thm:monotone}, if $\hat{s_1} \le s_1$, the
estimator $\hat{\lambda}$ is no less than $\lambda^*$. Otherwise, $\hat{s_1} >
s_1$ means $\hat{\lambda} < \lambda^*$. In addition, from~\eqref{eq:c1} we know
that $\hat{s_1}$ is a continuous function of $\hat{\lambda}$. Together with the
monotonicity given in Theorem~\ref{thm:monotone}, a bisection approach can be
employed to calculate $\lambda^*$. Algorithm~\ref{alg:sgpa} gives a detailed
description.

\begin{algorithm}[htb]
   \caption{Sparse Group Lasso Projection Algorithm}
   \label{alg:sgpa}
\begin{algorithmic}
  \REQUIRE $\bm v$, $s_1$, $s_2$
  \ENSURE \texttt{an optimal solution} $\bm x$
  \texttt{to the Sparse Group Projection Problem}
\end{algorithmic}
\textbf{Function} \texttt{SGLP}($\bm v$, $s_1$, $s_2$)
\begin{algorithmic}[1]
       \IF{$\|\bm x\|_1\le s_1$ \AND $\|\bm x\|_G\le s_2$}
           \RETURN $\bm v$
       \ENDIF

           \STATE $\bm x_{C_1}=\mcal{P}_1^{s_1}(\bm v)$
           \STATE $\bm x_{C_2}=\mcal{P}_G^{s_2}(\bm v)$
           \STATE $\bm x_{C_{12}}$= \texttt{bisec}($\bm v$, $s_1$, $s_2$)
           \IF{$\|\bm x_{C_1}\|_G\le s_2$} \RETURN $\bm x_{C_1}$
           \ELSIF{$\|\bm x_{C_2}\|_1\le s_1$} \RETURN $\bm x_{C_2}$
           \ELSE \RETURN $\bm x_{C_{12}}$
           \ENDIF
\end{algorithmic}

\textbf{Function} \texttt{bisec}($\bm v$, $s_1$, $s_2$)
\begin{algorithmic}[1]
           \STATE \texttt{Initialize $up$, $low$ and $tol$}
           \WHILE{$up-low > tol$}
           \STATE $\hat{\lambda} = (low+up)/2$
            \IF{\texttt{\eqref{eq:c2} has a solution} $\hat{\eta}$ \texttt{given} $v^{\hat{\lambda}}$}
                \STATE \texttt{calculate} $\h{s_1}$ \texttt{using} $\h{\eta}$ \texttt{and} $\hat{\lambda}$.
                   \IF{$\hat{s_1} \le s_1$} \STATE $up=\hat{\lambda}$
                   \ELSE \STATE $low=\hat{\lambda}$
                   \ENDIF
            \ELSE
                \STATE $up=\hat{\lambda}$
            \ENDIF
           \ENDWHILE
           \STATE $\lambda^*=up$
           \STATE \texttt{Solve~\eqref{eq:c2} to get} $\eta^*$
           \STATE \texttt{Calculate} $\bm x^*$ \texttt{from} $\lambda^*$ \texttt{and} $\eta^*$ \texttt{via}~\eqref{eq:sgl:sol}
        \RETURN $\bm x^*$
\end{algorithmic}
\end{algorithm}

\subsection{Solving Restricted version
 of~\eqref{eq:csglp:convex}}\label{subsec:equi:supp}
Finally, we modify the above procedures to compute the optimal
solution of the restricted problem~\eqref{eq:lm-prime}. To apply
the accelerated gradient method, we consider the following projection step:
\begin{equation}
 \label{eq:sglp:restricted}
 \begin{aligned}
  &\underset{x}{\text{minimize}} && \frac{1}{2}\|\bm x - \bm v\|^2_2 &&\\
   &\text{subject to} && \|\bm x^{T_1}\|_1 \le s_1 &&\quad(C_1)\\
  &                   && \|\bm x^{T_3}\|_G \le s_2 &&\quad(C_2).
 \end{aligned}
\end{equation}

Our first observation is: $T_3(\bm x)\subset T_1(\bm x)$, since if an element of
$\bm x$ lies in a group whose $L_2$-norm is less than $\tau$, then the absolute
value of this element must also be less than $\tau$. Secondly, from the
decomposable nature of the objective function, we conclude that:
\[
x^*_j = \left\{\begin{array}{rl}
v_j & \text{if } j\in (T_1)^c\\
v_j^{\lambda^*} &\text{if } j\in T_1\backslash T_3,
\end{array}\right.
\]
since there are no constraints on $x_j$ if it is outside $T_1$ and involves only
the $L_1$-norm constraint if $j\in T_1\backslash T_3$. Following routine
calculations as in~\cite{duchi2008efficient}, we obtain the following results
similar to~\eqref{eq:c1} and~\eqref{eq:c2}:
\begin{eqnarray}
 s_1 = \sum_{i\in T_2}\max\{\|\bm v_{G_i}^{\lambda^*}\|_2-\eta^*,0\}\frac{\|\bm v_{G_i}^{\lambda^*}\|_1}{\|\bm v_{G_i}^{\lambda^*}\|_2} + \sum_{j\in T_1\backslash T_3}v_j^{\lambda^*} \label{eq:c1:restricted}   & &\\
 s_2 = \sum_{i\in T_2}\max\{\|\bm v_{G_i}^{\lambda^*}\|_2 - \eta^*,0\}.\label{eq:c2:restricted} & &
\end{eqnarray}

Based on~\eqref{eq:c1:restricted} and \eqref{eq:c2:restricted}, we design a
similar bisection approach to compute $\lambda^*$ and thus $(\bm x^*)^{T_3}$, as
in Algorithm~\ref{alg:sgpa}. Details are deferred to the supplement.

\section{Significance}~\label{sec:significance}
 This section is devoted to a brief discussion of advantages of our work
statistically and computationally. Moreover, it explains why the
proposed method is useful to perform efficient and interpretable feature
selection with a given natural group structure.

\noindent \textbf{Interpretability.} The parameters in~\eqref{eq:csglp1} are highly
interpretable in that $s_1$ and $s_2$ are upper bounds of the
number of nonzero elements as well as that of groups. This is
advantageous, especially in the presence of certain prior knowledge regarding
the number of features and/or that of groups. However, such an interpretation
vanishes with convex methods such as lasso or sparse group
lasso, in which incorporating such prior knowledge
often requires repeated trials of different parameters.

\noindent \textbf{Parameter tuning.} Typically, tuning parameters for good generalization
 usually requires considerable amount work due to a large number of choices of
 parameters.  However, tuning in~\eqref{eq:csglp:nonconvex} may search through
 integer values in a bounded range, and can be further simplified when certain
 prior knowledge is available. This permits more efficient tuning than its
 regularization counterpart. Based on our limited experience, we note that
 $\tau$ does not need to be tuned precisely as we may fix at some small values.

\noindent \textbf{Performance and Computation.} Although our model~\eqref{eq:csglp1}
 is proposed as a computational surrogate of the ideal $L_0$-method, its
 performance can also be theoretically guaranteed, i.e., consistent feature
 selection can be achieved. Moreover, the computation of our model is much more
 efficient and applicable to large-scale applications.

\section{Empirical Evaluation}\label{sec:exp}

This section performs numerical experiments to evaluate the proposed methods
in terms of the efficiency and accuracy of sparse group feature selection.
Evaluations are conducted on a PC with i7-2600 CPU, $8.0$ GB
memory and 64-bit Windows operating system.

\subsection{Evaluation of Projection Algorithms}
Since the DC programming and the accelerated gradient methods are both standard,
the efficiency of the proposed nonconvex formulation~\eqref{eq:csglp1} depends
on the projection step in~\eqref{eq:sglp}. Therefore, we focus on evaluating the
projection algorithms and comparing with two popular projection algorithms:
Alternating Direction Multiplier Method (ADMM)~\cite{boyd2011distributed} and
Dykstra's projection algorithm~\cite{combettes2010proximal}. We provide a
detailed derivation of adapting these two algorithms to our formulation in the
supplement.

To evaluate the efficiency, we first generate the vector $\bm v$ whose
entries are uniformly distributed in $[-50, 50]$ and the dimension of $\bm v$, denoted as $p$, is
chosen from the set $\{10^2, 10^3, 10^4, 10^5, 10^6\}$. Next we partition the
vector into $10$ groups of equal size. Finally, $s_2$ is set to $5\log(p)$ and $s_1$, the radius of the
$L_1$-ball, is computed by $\frac{\sqrt{10}}{2}s_2$
(motivated by the fact that $s_1 \le \sqrt{10}s_2$).

For a fair comparison, we run our projection algorithm until
converge and record the minimal objective value as $f^*$.  Then we
run ADMM and Dykstra's algorithm until their objective values become close to
ours. More specifically, we terminate their iterations as soon as
$f_{\text{ADMM}} - f^* \le 10^{-3}$ and $f_{\text{Dykstra}} - f^* \le 10^{-3}$,
where $f_{\text{ADMM}}$ and $f_{\text{Dykstra}}$ stand for the objective value
of ADMM and Dykstra's algorithm respectively. Table~\ref{tab:projection:time}
summarizes the average running time of all three algorithms over 100 replications.

\begin{table}
\caption{Running time (in seconds) of Dykstra's, ADMM and our projection algorithm. All three algorithms are averaged over 100 replications.}\label{tab:projection:time}
\vskip 0.15in
\begin{center}
 \begin{small}
   \begin{sc}
     \begin{tabular}{c||ccccc}
 \hline
Methods & $10^2$  & $10^3$  & $10^4$  & $10^5$  & $10^6$  \\ \hline\hline
Dykstra &0.1944 &0.5894 &4.8702 &51.756 &642.60 \\
ADMM &0.0519 &0.1098 &1.2000 &26.240 &633.00 \\
ours & $< 10^{-7}$ &0.0002 &0.0051 &0.0440 &0.5827 \\
 \hline
   \hline
\end{tabular}
 \end{sc}
 \end{small}
\end{center}
\vskip -0.1in
\end{table}

Next we demonstrate the accuracy of our projection algorithm. Toward this end, the general convex optimization toolbox
CVX~\cite{cvx2011grant} is chosen as the baseline. Following the same strategy
of generating data, we report the distance (computed from the Euclidean norm
$\|\cdot\|_2$) between optimal solution of the three projection algorithms and
that of the CVX. Note that the projection is strictly convex with a unique
global optimal solution.

For ADMM and Dykstra's algorithm, the termination criterion is that the relative
difference of the objective values between consecutive iterations is less than a
threshold value. Specifically, we terminate the iteration if $|f(\bm x_{k-1})-f(\bm
x_{k})| \le 10^{-7} f(\bm x_{k-1})$. For our projection algorithm, we set the
$tol$ in Algorithm~\ref{alg:sgpa} to be $10^{-7}$. The results are summarized in
Table~\ref{tab:projection:accu}. Powered by second-order optimization
algorithms, CVX can provide fast and accurate solution for problems of moderate
size but would suffer from great computational burden for large-scale
ones. Therefore we only report the results up to $5,000$ dimensions.

\begin{table}[tbh]
 \caption{Distance between the optimal solution of projection algorithms and that of the CVX. All the results are averaged over 100 replications.}
 \label{tab:projection:accu}
\vskip 0.15in
 \begin{center}
   \begin{small}
       \begin{tabular}{c||cccccc}\hline
 Methods & 50 & 100 & 500 & 1000 & 5000 \\ \hline\hline
 Dykstra &9.00 &9.81 &11.40 &11.90 &12.42  \\
 ADMM &0.64 & 0.08  &3.6e-3 &6.3e-3 &1.3e-2 \\
 ours &1.4e-3 & 1.1e-3  &1.2e-3 &1.7e-3 &7.3e-3 \\ \hline     \hline
 \end{tabular}
   \end{small}
 \end{center}
\vskip -0.1in
\end{table}

From Tables~\ref{tab:projection:time} and~\ref{tab:projection:accu}, we note
that both ADMM and our algorithm yield more accurate solution than that of
Dykstra's. For projections of moderate size, all three algorithms perform
well. However, for large-scale ones, our advantage on efficiency is evident.

\subsection{Performance on Synthetic Data}\label{subsec:perf:synthetic}
\subsubsection{Experimental Setup}
We generate a $60\times 100$ matrix $\bm A$, whose entries follow i.i.d standard
normal distribution. The $100$ features (columns) are partitioned into $10$
groups of equal size. The ground truth vector $\bm x_0$ possesses nonzero elements
only in $4$ of the $10$ groups. To further enhance sparsity, in each nonzero
group of $\bm x_0$, only $t~(t\le 10)$ elements are nonzero, where $t$ is
uniformly distributed from $[1,5]$. Finally $\bm y$ is generated according to $\bm A \bm x_0 + \bm z$ with $\bm z$ following distribution $\mcal{N}(0,
0.5^2)$, where $\bm A$ and $\bm y$ are divided into training and testing set of equal size.

We fit our method to the training set and compare with lasso, group lasso and
sparse group lasso. The tuning parameters of the convex methods are selected from
$\{0.01, 0.1, 1, 10\}$, whereas for our method, the number of nonzero groups is
selected from the set $\{2, 4, 6, 8\}$ and the number of features is chosen
from $\{2s_2, 4s_2, 6s_2, 8s_2\}$. Leave-one-out cross-validation is conducted
over the training set for choosing the best tuning parameter for all the methods.
\subsubsection{Results and Discussions}

We use following metrics for evaluation:
\begin{itemize}
\item Estimation error: $\|\hat{\bm x}- \bm x_0\|_2^2 $
\item Prediction error: $\|\bm A \hat{\bm x} - \tilde{\bm y}\|_2^2 $
\item Group precision: ${|T_2(\hat{\bm x}) \cap T_2(\bm x_0)|}/{|T_2(\hat{\bm x})|}$
\item Group recall: ${|T_2(\hat{\bm x}) \cap T_2(\bm x_0)|}/{|T_2(\bm x_0)|}$
\end{itemize}
where $\hat{\bm x}$ is the estimator obtained from~\eqref{eq:csglp1} and
$\tilde{\bm y}$ is an independent vector following the same distribution as $\bm
y$. The group precision and recall demonstrate the capability of recovering the
group structure from data. We report the results in Table~\ref{tab:syn:ret} and
observe that our model generally exhibits better performance. Note that although
our model does not provide the best result on the metric of group recall, the
group precision of our model is significantly better than the others, illustrating the fact that the three convex methods recover more redundant
groups.

\begin{table}[tbh]
 \caption{Comparison of Performance on synthetic data, where glasso stands for the group lasso and sglasso denotes sparse group lasso. All the results are averaged for 100 replications.}\label{tab:syn:ret}
\vskip 0.15in
 \begin{center}
   \begin{small}
     \begin{sc}
           \begin{tabular}{c||cccc}\hline
 Methods & Esti. & Pred. & Prec. & Rec. \\ \hline \hline
 lasso & 4.7933 & 151.05 & 0.5212 & \textbf{0.8700} \\
 glasso & 8.1230 & 244.53 & 0.5843 & 0.7575 \\
 sglasso & 4.7649 & 151.29 & 0.5215 & 0.8675\\
 ours & \textbf{4.6617} & \textbf{142.18} & \textbf{0.7848} & 0.6450 \\
 \hline     \hline
 \end{tabular}
      \end{sc}
   \end{small}
 \end{center}
\vskip -0.1in
\end{table}

\subsection{Performance on Real-world Application}\label{subsec:perf:real}
Our method is further evaluated on the application of examining Electroencephalography (EEG) correlates
of genetic predisposition to alcoholism~\cite{Frank+Asuncion:2010}.
EEG records the brain's spontaneous electrical activity by measuring the voltage fluctuations over multiple electrodes placed on the
scalp. This technology has been widely used in clinical diagnosis, such as coma, brain death and genetic predisposition
to alcoholism. In fact, encoded in the EEG data is a certain
group structure, since each electrode records the electrical activity of a certain region of the scalp. Identifying and utilizing such spatial information has the potential of increasing stability of a prediction.

The training set contains $200$ samples of $16384$ dimensions, sampled from 64 electrodes placed on
subject's scalps at 256 Hz (3.9-msec epoch) for 1 second. Therefore, the data
can naturally be divided into 64 groups of size $256$. We apply the lasso,
group lasso, sparse group lasso and our method on the training set and adapt the
5-fold cross-validation for selecting tuning parameters. More
specifically, for lasso and group lasso, the candidate tuning parameters are
specified by 10 parameters\footnote{$\lambda_{\text{lasso}} =
 \text{logspace}(10^{-3}, 1)$,
 $\lambda_{\text{glasso}}=\text{logspace}(10^{-2}, 1)$} sampled using the
logarithmic scale from the parameter spaces, while for the sparse group lasso,
the parameters form a $10\times 10$ grid\footnote{The product space of
 $\lambda_{\text{lasso}}\times \lambda_{\text{glasso}}$}, sampled from the
parameter space in logarithmic scale. For our method, the number of groups is
selected from the set: $s_2 = \{30, 40, 50\}$ and $s_1$, the number of
features is chosen from the set $\{50s_2, 100s_2, 150s_2\}$. The accuracy of
classification together with the number of selected features and groups over a
test set, which also contains $200$ samples, are reported in
Table~\ref{tab:eeg}. Clearly our method achieves the best performance
of classification with the least number of groups. Note that, although lasso's
performance is almost as good as ours with even less features, however, it fails
to identify the underlying group structure in the data, as revealed by the fact all $64$ groups are selected.

\begin{table}[thb]
 \caption{Comparison of performance on EEG data, where glasso stands for group lasso and sglasso denotes sparse group lasso.} \vspace{0.05in}
 \label{tab:eeg}
 \centering
 \begin{tabular}{c||ccc}\hline
   Methods & Accuracy & \# Feature &\# Group \\ \hline \hline
   lasso &67.0  &2068  & 64  \\
   glasso &62.5  &8704  &34 \\
   sglasso &65.5 &4834 &61 \\
   ours &68.0 &3890 &25 \\
   \hline
   \hline
 \end{tabular}
\end{table}

\section{Conclusion and Future Work}\label{sec:conclusion}

This paper expands a nonconvex paradigm into sparse group feature selection. In
particular, theoretical properties on the accuracy of selection and parameter
estimation are analyzed. In addition, an efficient optimization scheme is
developed based on the DC programming, accelerated gradient method and efficient
projection. The efficiency and efficacy of the proposed method are validated on both synthetic data and real-world applications.

The proposed method will be further investigated on real-world applications involving the group structure. Moreover, extending the proposed model to multi-modal multi-task learning~\cite{zhang2011multimodal} is another promising direction.

{ \small
\bibliography{sgfs}
\bibliographystyle{plain}
}

\newpage
\input{supplemental.tex}

\end{document}

%% file: theory.tex
This section investigates theoretical aspects of the proposed method. More
specifically, we demonstrate that the oracle estimator $\hat{\bm x}^{o}$, the
least squares estimator based on the true model, can be reconstructed.  As a
result, consistent selection as well as optimal parameter estimation can be
achieved.

For better presentation, we introduce some notations that would be only utilized
in this section. Let $C=(G_{i_1},\cdots,G_{i_k})$ be the collection of groups
that contain nonzero elements. Let $A_{G_j} = A_{G_j}(\bm x)$ and $A=A(\bm x)$
denote the indices of nonzero elements of $\bm x$ in group $G_j$ and in entire
$\bm x$ respectively. Define
\[
\mathcal S_{j,i}=\{\bm x \in
\mathcal S: (A_C, C) \neq (A_{C^0},C^0), |A|=j, |C|=i\},
\]
where $\mcal{S}$ is the feasible region of~\eqref{eq:csglp1} and $C^0$
represents the true nonzero groups.

The following assumptions are needed for obtaining consistent reconstruction of
the oracle estimator:

\begin{assump}[Separation condition] \label{assump:separation}
Define
\[
C_{\min}(\bm x^0) = \inf_{\bm x\in \mathcal S} \frac{- \log (1- h^2(\bm x,\bm x^0))}{\max(|C^0 \setminus C|,1)},
\]
then for some constant $c_1>0$,
\begin{equation*}
C_{\min}(\bm x^0 )  \geq  c_1\frac{ \log |G|+\log s_1^0}{n},
\end{equation*}
where  \[
h(\bm x,\bm x^0)=\big(\frac{1}{2}\int (g^{1/2}(\bm x,y)-g^{1/2}(\bm
x^0,y))^2
d \mu(y) \big)^{1/2}
\] is the Hellinger-distance for densities with respect to a dominating measure $\mu$.
\end{assump}


\begin{assump}[Complexity of the parameter space]\label{assump:A}
For some constants
$c_0>0$ and any $0<t< \varepsilon \leq 1$,
\[
H(t, {\cal F}_{j,i}) \leq c_0
\max((\log (|G| + s_1^0))^2,1)|{\mathcal B}_{j,i}|
\log (2 \varepsilon/t),
\]
where ${\cal B}_{j,i}={\cal S}_{j,i} \cap \{\bm x \in h(\bm x,\bm x^0) \leq 2
\varepsilon\}$ is a local parameter space and ${\cal F}_{j,i}=\{g^{1/2}(\bm
x,y): \bm x \in {\cal B}_{j,i} \}$ is a collection of square-root
densities. $H(\cdot,{\cal F})$ is the bracketing Hellinger metric entropy of
space $\cal F$~\cite{kolmogorov1961entropy}.
\end{assump}

\begin{assump}\label{assump:B}
For some positive constants $d_1, d_2, d_3$ with $d_1>10$,
\[
-\log (1 - h^2(\bm x,\bm x^0) )\geq  - d_1\log (1- h^2(\bm x^{\tau},
\bm x^0) )- d_3  \tau^{d_2} p,
\]
where $\bm x^{\tau} = (x_1 I(|x_1| \geq \tau),\cdots, x_p
I(|x_p| \geq \tau))$.
\end{assump}

With the above assumptions hold, we can conclude the following non-asymptotic
probability error bound regarding the reconstruction of the oracle estimator
$\hat{\bm x}^o$.


\begin{theorem} \label{thm:main:theory} Suppose that Assumptions~\ref{assump:A}
  and~\ref{assump:B} hold. For a global minimizer of~\eqref{eq:csglp1} $\hat{\bm
    x}$ with $(s_1,s_2)=(s^0_1,s^0_2)$ and $\tau \leq \big(\frac{(d_1 -
    10)C_{\min} (\bm x^0) }{d_3d} \big)^{1/d_2},$ the following result hold:
\[
  \mathbb P\Big(\hat{\bm x} \neq \hat{\bm x}^{o}\Big) \leq \exp \Big(- c_2 n C_{\min}(\bm x^0) + 2(\log |G| + \log s_1^0)\Big).
\]
Moreover, with Assumption~\ref{assump:separation} hold, $\mathbb P\Big(\hat{\bm x}= \hat{\bm x}^{o}\Big) \rightarrow 1$ and
\[
  E h^2(\hat{\bm x},\bm x^o) =(1+o(1)) \max(E h^2(\bm{\hat x}^o,{\bm x}^0),
\frac{s_1^0}{n})
\]
as $n \rightarrow \infty$, $|G| \rightarrow \infty$.
\end{theorem}

Theorem~\ref{thm:main:theory} states that the oracle estimator $\hat{\bm x}^o$
can be accurately reconstructed, which in turn yields feature selection
consistency as well as the recovery of the performance of the oracle estimator
in parameter estimation. Moreover, according to
Assumption~\ref{assump:separation}, such conclusion still holds when $s_1^0 |G|$
grows in the order of $\exp(c_1^{-1} n C_{\min})$ . This is in contrast to
existing conclusions on consistent feature selection, where the number of
candidate features should be no larger than $\exp(c^* n)$ for some
$c^*$~\cite{zhao2006model}. In this sense, the number of candidate features is
allowed to be much larger when an additional group structure is incorporated,
particularly when each group contains considerable redundant features. 

To our knowledge, our theory for the grouped selection is the first of this
kind. However, it has a root in feature selection. The large deviation approach
used here is applicable to derive bounds for feature selection consistency. In
such a situation, the result agrees with the necessary condition for feature
selection consistency for any method, except for the constants independent of
the sample size~\cite{shen2012likelihood}. In other words, the required
conditions are weaker than those for
$L_1$-regularization~\cite{van2009conditions}. The use of the Hellinger-distance
is mainly to avoid specifying a sub-Gaussian tail of the random error. This
means that the result continues to hold even when the error does not have a
sub-Gaussian tail.

%% file: supplemental.tex







\setlength{\textwidth}{16.5cm}
\setlength{\textheight}{22.5cm}
\setlength{\oddsidemargin}{0.0in}
\setlength{\evensidemargin}{0.0in}
\setlength{\topmargin}{-0.0in}
\setlength{\parskip}{7pt plus 2pt minus 2pt}

\begin{center}
\Large Supplementary Material for paper: \\Efficient Sparse Group Feature Selection via \\ Nonconvex Optimization
\end{center}

\setcounter{section}{0}
\section{Proof of Theorem~\ref{thm:main:theory}}
The proof uses a large deviation probability inequality of~\cite{wong1995probability} to treat one-sided log-likelihood
ratios with constraints.  

  Let $\mathcal S = \big \{\bm x^{\tau}: \|\bm x^{\tau}\|_0 \leq s_1^0, 
\|\bm x^{\tau}\|_{0,G} \leq s_2^0 \big\}$, $\|\bm x\|_0=\sum_{j=1}^p 
I(|x_j| \neq 0)$ is the $L_0$-norm of $\bm x$, and 
$\|\bm x\|_{0,G}=\sum_{j=1}^{|G|} I(\|\bm x_j\|_2 \neq 0)$ is the 
$L_0$-norm over the groups. 
Now we partition $\mathcal S$.  Note that for $C \subset 
(G_1,\cdots,G_{|G|})$, it can be partitioned into $C=(C \setminus C^0) 
\cup (C \cap C^0)$.  Then 
$$\mathcal S=  \bigcup_{i = 0}^{s^0_2} \bigcup_{C\in \mathcal B_i} 
\mathcal S_{A_C,C},
$$ 
where $S_{A_C,C}=
\big \{\bm x^{\tau} \in \mathcal S: C(\bm x) = C=(G_{i_1},\cdots,
G_{i_k}), \sum_j |A_{G_j}| \leq s_1^0\big\}$, 
and $\mathcal B_{i}= \{C \neq C_0: |C^0 \setminus C| = i, |C| \leq s_2^0\}$, 
with $|\mathcal B_{i}|=\binom{s^0_2}{s^0_2 - i} 
\sum_{j = 0}^{i}\binom{|G|-s^0_2}{j}$; $i = 0,\cdots, s^0_2$.

To bound the error probability, let $L(\bm x)=-\frac{1}{2} \|\bm A \bm x- \bm
y\|^2$ be the likelihood. Note that
\[
\{\hat{\bm x} \neq \hat{\bm x}^{o}\} \subseteq 
\{L(\hat{\bm x}) - L(\hat{\bm x}^{o}) \geq 0\} \subseteq \{L(\hat{\bm x})- L(\bm x^0) \geq 0\}.
\]
This together with $\{\hat{\bm x}\neq\hat{\bm x}^{o} \} \subseteq \{\hat{\bm x} 
\in \mathcal S\} $ implies that 
\[
\{\hat{\bm x}\neq\hat{\bm x}^{o} \} \subseteq \{ L(\hat{\bm x})- L(\bm x^0) \geq
0\} \cap \{\hat{\bm x} \in \mathcal S\}.
\] 
Consequently, 
\begin{equation*}
  \begin{aligned}
   &  I \equiv P\big(\hat{\bm x}\neq\hat{\bm x}^{o}\big) \\
&\leq P \Big(  L(\hat{\bm x}) - L(\bm x^0)   \geq 0;\hat{\bm x} \in \mathcal S\Big) \\
 & \leq   \sum_{i = 1}^{s^0_2} \sum_{C \in {\mathcal B}_i}  
\sum_{S_{A_C,C}} P^* \Big( \sup_{\bm x \in  \mathcal S_{A_C,C}} \big(L(\bm x)
- L(\bm x^0) \big)\geq 0   \Big)  \\
 &\leq  \sum_{i = 1}^{s^0_2} \sum_{j=1}^{s_1^0}
\sum_{|C|=i,  |A_G| =j} 
 P^* \Big( \sup_{ \big \{ - \log (1- h^2(\bm x,\bm x^0) ) 
\geq\max(i,1)  C_{\min}(\bm x^0) -d_3 \tau^{d_2} p, 
\bm x \in \mathcal S_{A_C,C} \big\}} 
\big(L(\bm x)- L(\bm x^0) \big)\geq 0   \Big),    
  \end{aligned}
\end{equation*}
where $P^*$ is the outer measure and the last two inequalities use the fact that
$\mathcal S_{A_C,C} \subseteq\{\bm x \in \mathcal S_{A_C,C}:
\max(|C^0 \setminus C|, 1)C_{\min}(\bm x^0) \leq - 
\log (1- h^2(\bm x,\bm x^0) )\} \subseteq  
\{- \log (1 -  h^2 (\bm x, \bm x^0)) \geq d_1 \max(i,1) 
C_{\min}(\bm x^0) - d_3 \tau^{d_2} p\}$, under Assumption~\ref{assump:B}.

  For $I$, we apply Theorem 1 of~\cite{wong1995probability} to bound each term. Towards
this end, we verify their entropy
condition (3.1) for the local entropy over ${\mathcal S}_{A_C,C}$ for 
$|C|=1,\cdots,s^0_2$ and $|A|= 1,\cdots, s^0_1$.
Under Assumption~\ref{assump:A}
$\varepsilon=\varepsilon_{n,p}=(2 c_0)^{1/2} c_4^{-1} \log (2^{1/2}/c_3)
\log p(\frac{s^0_1}{n})^{1/2}$ satisfies there with respect to
$\varepsilon>0$, that is,
\begin{eqnarray}
 \sup_{\{0 \leq |A| \leq p_0\}} \int^{2^{1/2} \varepsilon}_{2^{-8}
\varepsilon^2} H^{1/2}(t/c_3, {\cal F}_{ji})d t \leq p_0^{1/2} 2^{1/2}
 \varepsilon \log (2/2^{1/2}c_3) \leq  c_4 n^{1/2} \varepsilon^2.
 \label{entropy}
 \end{eqnarray}
for some constant $c_3>0$ and $c_4>0$, say $c_3=10$ and 
$c_4=\frac{(2/3)^{5/2}}{512}$.
By Assumption~\ref{assump:A}, $C_{\min}(\bm x^0) \geq \varepsilon^2_{n,p_0,p}$ implies
(\ref{entropy}), provided that $s^0_1 \geq (2 c_0)^{1/2} c_4^{-1} \log (2^{1/2}/c_3)$.

Note that $|\mathcal B_{i}|=\binom{s^0_2}{s^0_2 - i}
\sum_{j = 0}^{i}\binom{|G|-s^0_2}{j} \leq (|G|(|G|-s^0_2)^i \leq
(|G|^2/4)^i$ by the binomial coefficients formula. Moreover,
$\sum_{j=1}^{s_1^0} 2^j i^j \leq i^{s_1^0}$, and  
$\sum_{j_1+\cdots+j_i=j} \binom{j}{j_1,\cdots j_i} 2^j
=(2 i)^j$
using the Multinomial Theorem.
By  Theorem 1 of~\cite{wong1995probability}, there exists 
a constant $c_2>0$, say $c_2=\frac{4}{27} \frac{1}{1926}$,
\begin{eqnarray*}
 I & \leq &  \sum_{i=1}^{s^0_2} |\mathcal B_i| \sum_{j=1}^{s_1^0} 
\sum_{(j_1,\cdots j_i)} \binom{j}{j_1,\cdots j_i} 2^{j_1} \cdots
2^{j_i} \exp \big(- c_2 n i C_{\min}(\bm x^0) \big) \\
& \leq &    \sum_{i=1}^{s^0_2} 
\exp \big(- c_2 n i C_{\min}(\bm x^0)+ 2 i(\log |G|+\log s_1^0) \big) \\
& \leq &  \exp \big(- c_2 n  C_{\min}(\bm x^0) + 
2 (\log |G|+\log s_1^0)\big).
\end{eqnarray*}

  Let $G=\{\hat{\bm x} \neq \hat{\bm x}^0\}$.
For the risk property, $E h^2(\hat{\bm x},\bm x^0)= 
E h^2(\bm{\hat x}^0, {\bm x}^0)+ E h^2(\hat{\bm x},
\bm x^0) I(G)$ is upper bounded by
\begin{eqnarray*}
 E h^2(\bm{\hat x},{\bm x}^0)
+ \exp \big(- c_2 n  C_{\min}(\bm x^0) +
2 (\log |G|+\log s_1^0)\big)=(1+o(1))
E h^2(\bm{\hat x}^0,{\bm x}^0),
\end{eqnarray*}
using the fact that $h(\hat{\bm x},\bm x^0) \leq 1$.
This completes the proof.

\section{Proof of Theorem~\ref{thm:monotone}}
We utilize an intermediate lemma from~\cite{bonnans1998optimization}:
\begin{lemma}\label{lm:unknown}
  Let $X$ be a metric space and U be a normed space. Suppose that for all $x\in X$,
  the function $\psi(x,\cdot)$ is differentiable and that $\psi(x, Y)$ and
  $D_Y\psi(x, Y)$ (the partial derivative of $\psi(x, Y)$ with respect to $Y$)
  are continuous on $X\times U$. Let $\Phi$ be a compact subset of $X$. Define
  the optimal value function as $\phi(Y) = \inf_{x\in\Phi}\psi(x, Y)$. The
    optimal value function $\phi(Y)$ is directionally differentiable. In
    addition, if for any $Y\in U$, $\psi(\cdot, Y)$ has a unique minimizer
    $x(Y)$ over $\Phi$, then $\phi(Y)$ is differentiable at $Y$ and the gradient
    of $\phi(Y)$ is given by $\phi'(Y) = D_Y\psi(x(Y), Y)$.
\end{lemma}

\begin{proof}[Proof of Theorem~\ref{thm:monotone}]
For the proof, an intermediate lemma  will
be used, with its details given in the Appendix. Since both constraints are
  active, if $(x,\lambda,\eta) = \text{SGLP}(v, s_1, s_2)$, then $x$ and
  $\lambda$ are also the optimal solutions to the following problem:
\begin{equation*}
  \begin{aligned}
    &\underset{\lambda}{\text{maximize}} ~\underset{x\in X}{\text{minimize}} &&
    \psi(x,\lambda) = \frac{1}{2}\|x-v\|_2^2 + \lambda(\|x\|_1 - s_1) ,
  \end{aligned}
\end{equation*}
where $X=\{x: \|x\|_G \le s_2\}$. By Lemma~\ref{lm:unknown},
$\phi(\lambda)=\inf_{x\in X}\psi(x,\lambda)$ is differentiable with the
derivative given by $\|x\|_1$. In addition, as a pointwise infimum of a concave
function, so does $\phi(\lambda)$~\cite{boyd2004convex} and
its derivative, $\|x\|_1$, is non-increasing. Therefore $s_1=\|x\|_1$ is
non-decreasing as $\lambda$ becomes smaller. This completes the proof.
\end{proof}

\section{Algorithm for Solving~\eqref{eq:sglp:restricted}}
We give a detailed description of algorithm for solving the restricted
projection~\eqref{eq:sglp:restricted} in Algorithm~\ref{alg:rsgpa}.

\begin{algorithm}[H]
    \caption{Restricted Sparse Group Lasso Projection Algorithm}
    \label{alg:rsgpa}
 \begin{algorithmic}
   \REQUIRE $\bm v$, $s_1$, $s_2$, $T_1$, $T_3$
   \ENSURE \texttt{an optimal solution} $\bm x$
   \texttt{to the Restricted Sparse Group Projection Problem~\eqref{eq:sglp:restricted}}
 \end{algorithmic}
\textbf{Function} \texttt{RSGLP}($\bm v$, $s_1$, $s_2$, $T_1$, $T_3$)
 \begin{algorithmic}[1]
        \IF{$\|\bm x^{T_1}\|_1\le s_1$ \AND $\|\bm x^{T_3}\|_G\le s_2$}
            \RETURN $\bm v$
        \ENDIF
        \STATE $\bm x_{C_1}^{(T_1)^c}=\bm v^{(T_1)^c}$, $\bm
        x_{C_1}^{T_1}=\mcal{P}_1^{s_1}(\bm v^{T_1})$ 
        \STATE $\bm
        x_{C_2}^{(T_3)^c}=\bm v^{(T_3)^c}$, $\bm x_{C_2}^{T_3}=\mcal{P}_G^{s_2}(\bm
        v^{T_3})$ 
        \STATE $\bm x_{C_{12}}^{(T_1)^c} = \bm v^{(T_1)^c}$, $\bm x_{C_{12}}^{T_1}$= \texttt{bisec}($\bm v$, $s_1$, $s_2$, $T_1$, $T_3$)
        \IF{$\|\bm x_{C_1}^{T_3}\|_G\le s_2$} 
            \RETURN $\bm x_{C_1}$ 
        \ELSIF{$\|\bm x_{C_2}^{T_1}\|_1\le s_1$} 
            \RETURN $\bm x_{C_2}$ 
        \ELSE 
            \RETURN $\bm x_{C_{12}}$
        \ENDIF
 \end{algorithmic}

\textbf{Function} \texttt{bisec}($\bm v$, $s_1$, $s_2$, $T_1$, $T_3$)
 \begin{algorithmic}[1]
            \STATE \texttt{Initialize $up$, $low$ and $tol$}
            \WHILE{$up-low > tol$}
            \STATE $\hat{\lambda} = (low+up)/2$
             \IF{\texttt{\eqref{eq:c2:restricted} has a solution} $\hat{\eta}$ \texttt{given} $v^{\hat{\lambda}}$}
                 \STATE \texttt{calculate} $\h{s_1}$ \texttt{using} $\h{\eta}$ \texttt{and} $\hat{\lambda}$.
                    \IF{$\hat{s_1} \le s_1$} \STATE $up=\hat{\lambda}$
                    \ELSE \STATE $low=\hat{\lambda}$
                    \ENDIF
             \ELSE
                 \STATE $up=\hat{\lambda}$
             \ENDIF
            \ENDWHILE
            \STATE $\lambda^*=up$
            \STATE \texttt{Solve~\eqref{eq:c2:restricted} to get} $\eta^*$
            \STATE \texttt{Calculate} $(\bm x^*)^{T_1}$ \texttt{from} $\lambda^*$ \texttt{and} $\eta^*$.
         \RETURN $(\bm x^*)^{T_1}$
 \end{algorithmic}
 \end{algorithm}

\section{Accelerated Gradient Method}
\label{sec:sup:agm}
The AGM procedure is listed in Algorithms~\ref{alg:agm}, in which
$f(\bm x)$ is the objective function $\frac{1}{2}\|\bm A \bm x - \bm y\|_2^2$
with $\nabla f(\bm x)$ denotes its gradient at $\bm x$. In addition, $f_{L, \bm
  u}(\bm x)$ is the linearization of $f(\bm x)$ at $\bm u$ defined as follows:
\[
f_{L, \bm u}(\bm x) = f(\bm u) + \nabla f(\bm u)^T(\bm x - \bm u) + \frac{L}{2}\|\bm x-\bm u\|_2^2.
\]
\begin{algorithm}[H]
    \caption{Accelerated Gradient Method~\cite{nesterov2007gradient, beck2009fast} for~\eqref{eq:csglp:convex}}
    \label{alg:agm}
 \begin{algorithmic}[1]
    \REQUIRE $\bm A$, $\bm y$, $s_1$, $s_2$, $L_0$, $\bm x_0$,
    \ENSURE solution $\bm x$ to~\eqref{eq:csglp:convex}
    \STATE \textbf{Initialize}: $L_0$, $\bm x_1=\bm x_0$, $\alpha_{-1}=0$, $\alpha_0=1$, $t=0$.
    \REPEAT
    \STATE $t=t+1$, $\beta_t=\frac{\alpha_{t-2}-1}{\alpha_{t-1}}$, $\bm u_t=\bm x_t+\beta_t(\bm x_t- \bm x_{t-1})$
    \STATE \textbf{Line search}: \texttt{Find the smallest }$L=2^jL_{t-1}$ \texttt{ such that}
    \[
        f(\bm x_{t+1})\le f_{L, \bm u_t}(\bm x_{t+1}),
    \]
    \texttt{where} $\bm x_{t+1}=\text{SGLP}(\bm u_t-\frac{1}{L}\nabla f(\bm
    u_t), s_1, s_2)$ \STATE $\alpha_{t+1}= \frac{1+\sqrt{1+4\alpha_t^2}}{2}$,
    $L_t=L$.  \UNTIL{Converge} \RETURN $\bm x_t$
 \end{algorithmic}
\end{algorithm}

 \section{ADMM Projection algorithm}
 ADMM is widely chosen for its capability of decomposing coupled variables/constraints, which is exactly the case in our projection problem. Before applying ADMM, we transform~\eqref{eq:sglp} into an equivalent form as follows:
 \begin{equation*}
   \label{eq:sgpp:admm}
   \begin{aligned}
   &\underset{x}{\text{minimize}} && \frac{1}{2}\|\bm x - \bm v\|^2_2 \\
     &\text{subject to} && \|\bm u\|_1 \le s_1 \\
    &                   && \|\bm w\|_G \le s_2 \\
   &                    && \bm u = \bm x, \bm w = \bm x.
   \end{aligned}
 \end{equation*}
 The augmented Lagrangian is:
 \begin{equation*}
 \begin{aligned}
 \mathcal{L}(\bm x, \bm \lambda, \bm \eta) = \frac{1}{2}\|\bm x -\bm v\|^2_2 &+ \bm \lambda^T(\bm u- \bm x) + \bm \eta^T(\bm w - \bm x) \\
 &+ \frac{\rho}{2}(\|\bm u - \bm x\|_2^2 + \|\bm w - \bm x\|_2^2).
 \end{aligned}
 \end{equation*}
 Utilize the scaled form~\cite{boyd2011distributed}, i.e., let $\bm \lambda =
 \frac{\bm \lambda}{\rho}$, $\bm \eta = \frac{\bm \eta}{\rho}$, we can obtain an
 equivalent augmented Lagrangian:
 \begin{equation*}
 \begin{aligned}
 \mathcal{L}(\bm x, \bm \lambda, \bm \eta) = \frac{1}{2}\|\bm x - \bm v\|^2_2 &+ \frac{\rho}{2}(\|\bm x - \bm u - \bm \lambda\|_2^2 + \|\bm x - \bm w - \bm \eta\|_2^2) \\
 &- \frac{\rho}{2}(\|\bm \lambda\|_2^2+\|\bm \eta\|_2^2).
 \end{aligned}
 \end{equation*}

 Now we calculate the optimal $\bm x$, $\bm \lambda$ and $\bm \eta$ through
 alternating minimization. For fixed $\bm u$ and $\bm w$, the optimal $\bm x$
 possesses a closed-form solution:
 \[
 \bm x = \frac{1}{1+2\rho}\left(\bm v + \rho(\bm u + \bm \lambda + \bm w + \bm \eta)\right).
 \]

 For fixed $\bm x$ and $\bm u$, finding the optimal $\bm w$ is a group lasso
 projection:
 \begin{equation}
 \label{eq:admm:groupLassoProjection}
  \begin{aligned}
     &\underset{\bm w}{\text{minimize}} && \frac{1}{2}\|\bm w - (\bm x - \bm \eta)\|^2_2 \\
     &\text{subject to} && \|\bm w\|_G\le s_2
   \end{aligned}
 \end{equation}

 For fixed $\bm x$ and $\bm w$, finding the optimal $\bm u$ amounts to solve an
 $L_1$-ball projection:
 \begin{equation}
   \label{eq:admm:l1Projection}
   \begin{aligned}
     &\underset{\bm u}{\text{minimize}} && \frac{1}{2}\|\bm u - (\bm x - \bm \lambda)\|_2^2\\
     &\text{subject to}   && \|\bm u\|_1 \le s_1.
   \end{aligned}
 \end{equation}
The update of multipliers is standard as follows:
\begin{equation}
\begin{aligned}
\bm \lambda &= \bm \lambda + \bm u - \bm x\\
\bm \eta &= \bm \eta + \bm w - \bm x
\end{aligned}
\end{equation}
Algorithm~\ref{alg:admm} summarizes the above procedure. Note that, the value of
the penalty term $\rho$ is fixed in Algorithm~\ref{alg:admm}. However, in our
implementation, we increase $\rho$ whenever necessary to obtain faster
convergence.
\begin{algorithm}[h]
    \caption{ADMM~\cite{boyd2011distributed} for~\eqref{eq:sglp}}
    \label{alg:admm}
 \begin{algorithmic}[1]
    \REQUIRE $\bm v$, $s_1$, $s_2$
    \ENSURE \texttt{an optimal solution} $x$ \texttt{to}~\eqref{eq:sglp}
    \STATE \textbf{Initialize}: $\bm x_0$, $\bm u_0$, $\bm w_0$, $\bm \lambda_0$, $\bm \eta_0$, $t=0$, $\rho > 0$
    \REPEAT
    \STATE $t=t+1$
    \STATE $\bm x_{t}= \frac{1}{1+2\rho}\left(\bm v + \rho(\bm u_{t-1} + \bm \lambda_{t-1} + \bm w_{t-1} + \bm \eta_{t-1})\right)$
    \STATE $\bm w_{t} = \mcal{P}_G^{s_2}(\bm x_t - \bm \eta_{t-1})$
    \STATE $\bm u_{t} = \mcal{P}_1^{s_1}(\bm x_t - \bm \lambda_{t-1})$
    \STATE $\bm \lambda_t = \bm \lambda_{t-1} + \bm u_t - \bm x_t$, $\bm \eta_t = \bm \eta_{t-1} + \bm w_t - \bm x_t$.
    \UNTIL{\texttt{Converge}}
    \RETURN $\bm x_t$
 \end{algorithmic}
\end{algorithm}

\section{Dykstra's Algorithm}
Dykstra's algorithm is a general scheme to compute the projection onto
intersections of convex sets. It is carried out by taking Euclidean projections
onto each convex set alternatively in a smart way and is guaranteed to converge
for least squares objective function~\cite{combettes2010proximal}. The details
of applying Dykstra's Algorithm to our projection problem are listed in
Algorithm~\ref{alg:dykstra}.

\begin{algorithm}[h]
    \caption{Dykstra's Algorithm~\cite{combettes2010proximal} for~\eqref{eq:sglp}}
    \label{alg:dykstra}
 \begin{algorithmic}[1]
    \REQUIRE $\bm v$, $s_1$, $s_2$
    \ENSURE \texttt{an optimal solution} $x$ \texttt{to}~\eqref{eq:sglp}
    \STATE \textbf{Initialize}: $\bm x_0 = \bm v$, $\bm p_0 = \bm 0$, $\bm q_0 = \bm 0$, $t = 0$
    \REPEAT
    \STATE $t=t+1$
    \STATE $\bm y_{t-1}= \mcal{P}_G^{s_2}(\bm x_{t-1}+\bm p_{t-1})$
    \STATE $\bm p_t = \bm x_{t-1} + \bm p_{t-1} - \bm y_{t-1}$
    \STATE $\bm x_t = \mcal{P}_1^{s_1}(\bm y_{t-1} + \bm q_{t-1})$
    \STATE $\bm q_{t} = \bm y_{t-1} + \bm q_{t-1} - \bm x_t$
    \UNTIL{\texttt{Converge}}
    \RETURN $\bm x_t$
 \end{algorithmic}
\end{algorithm}